\documentclass{article}

\usepackage{PRIMEarxiv}

\usepackage[utf8]{inputenc} 
\usepackage[T1]{fontenc}    
\usepackage{hyperref}       
\usepackage{url}            
\usepackage{booktabs}       
\usepackage{amsfonts}       
\usepackage{nicefrac}       
\usepackage{microtype}      
\usepackage{lipsum}
\usepackage{fancyhdr}       
\usepackage{graphicx}       
\graphicspath{{media/}}     

\usepackage[utf8]{inputenc} 
\usepackage[T1]{fontenc}    
\usepackage{hyperref}       
\usepackage{url}            
\usepackage{booktabs}       
\usepackage{amsfonts}       
\usepackage{nicefrac}       
\usepackage{microtype}      
\usepackage{lipsum}
\usepackage{amssymb}
\usepackage{fancyhdr}       
\usepackage{graphicx}       
\usepackage{xcolor}
\usepackage{enumerate}
\usepackage{textcomp}
\usepackage{mathrsfs}
\usepackage{amsthm}
\usepackage{bbm}
\usepackage{epigraph}
\usepackage{etoolbox}
\usepackage{aligned-overset}
\usepackage{thmtools, thm-restate}
\usepackage{proof-at-the-end}

\usepackage[capitalize,noabbrev]{cleveref}

\theoremstyle{plain}
\newtheorem{theorem}{Theorem}[section]

\newtheorem{lemma}[theorem]{Lemma}

\theoremstyle{definition}

\theoremstyle{remark}

\usepackage[textsize=tiny]{todonotes}

\def\KL{\mathbf{d}_{\mathrm{KL}}}

\def\normal{\mathcal{N}}

\def\proxytheta{\tilde{\theta}}
\def\relu{{\rm ReLU}}

\def\E{\mathbb{E}}

\def\H{\mathbb{H}}

\def\I{\mathbb{I}}
\def\Pr{\mathbb{P}}

\def\L{\mathbb{L}}
\def\1{\mathbf{1}}

\def\relu{\text{ReLU}}

\newcommand{\Lc}{\mathcal{L}}

\newcommand{\sphere}{\mathbb{S}^{d-1}}
\newcommand{\Z}{\mathbb{Z}}
\usepackage{hyperref}

\DeclareMathOperator*{\argmin}{arg\,min}

\usepackage[mathscr]{euscript}
\hypersetup{
    colorlinks=true,
    linkcolor=blue,
    filecolor=magenta,      
    urlcolor=purple,
}
\usepackage{xcolor}

\newcount\Comments
\newcommand{\kibitz}[2]{\ifnum\Comments=1{\textcolor{#1}{\textsf{\footnotesize #2}}}\fi}
\usepackage{color}
\definecolor{darkred}{rgb}{0.7,0,0}
\definecolor{darkgreen}{rgb}{0.0,0.5,0.0}
\definecolor{darkblue}{rgb}{0.0,0.0,0.5}
\definecolor{teal}{rgb}{0.0,0.5,0.5}

\usepackage{natbib}
\title{Information-Theoretic Foundations for
\\ Neural Scaling Laws
}

\author{
  Hong Jun Jeon \\
  Computer Science \\
  Stanford University \\
  Stanford, CA\\
  \texttt{hjjeon@stanford.edu} \\
   \And
  Benjamin Van Roy \\
  Stanford University\\
  Stanford, CA\\
  \texttt{bvr@stanford.edu} \\
}

\begin{document}
\maketitle

\begin{abstract}
Neural scaling laws aim to characterize how out-of-sample error behaves as a function of model and training dataset size.  Such scaling laws guide allocation of a computational resources between model and data processing to minimize error.  However, existing theoretical support for neural scaling laws lacks rigor and clarity, entangling the roles of information and optimization.  In this work, we develop rigorous information-theoretic foundations for neural scaling laws.  This allows us to characterize scaling laws for data generated by a two-layer neural network of infinite width.  We observe that the optimal relation between data and model size is linear, up to logarithmic factors, corroborating large-scale empirical investigations.  Concise yet general results of the kind we establish may bring clarity to this topic and inform future investigations.
\end{abstract}

\keywords{Information Theory \and Neural Scaling Laws}

\section{Introduction}
In recent years, foundation models have grown immensely, with some embodying trillions of trainable parameters.  While larger models have in general produced better results, they also require much more compute to train. It has become impractical to perform hyperparameter sweeps at the scale of these modern models.  This has required bypassing the practice of tuning hyperparameters via extensive trial and error, as was previously common in deep learning.

Among other things, hyperparameters control $1)$ the size, measured in terms of the parameter count $p$, of the neural network model and $2)$ the number $T$ of training tokens.  If each parameter is adjusted in response to each token then the computational requirements of training scale will the product of these two quantities.  For any compute budget $C$, one should carefully balance between $p$ and $T$. Too few training tokens leads to model estimation error, while too few parameters gives rise to mispecification error. As evaluating performance across multiple choices of $p$ and $T$ becomes computationally prohibitive at scale, alternative kinds of analysis are required to guide allocation of computational resources.

\citet{nlm} and \citet{chinchilla} have proposed the following procedure for allocating a large compute budget: $1)$ Evaluate test errors of models produced using various small compute budgets $C$ with many different allocations to parameters $p$ versus training tokens $T$. $2)$ Extrapolate to estimate the relation between $p$ and $T$ for large $C$. $3)$ Extrapolate to estimate the relation between $p$ and $T$ for large $C$.

To give a sense of scales involved here, \citet{chinchilla} evaluate test errors across ``small'' models for which $p\times T$ ranges from around $10^{18}$ to $10^{22}$ and extrapolates out to ``large'' models at around $10^{24}$. \citet{nlm} and \citet{chinchilla} each extrapolate based on a hypothesized \emph{scaffolding function}.  \citet{nlm} guess a scaffolding function based on results observed in small scale experiments. \citet{chinchilla} carry out an informal and somewhat speculative mathematical analysis to guide their choice (see their Appendix D).

The analysis of \citet{chinchilla} is somewhat generic rather than specialized to the particular neural network architecture used in that paper. In this paper, building on the work of \citet{JeonNeurips2022,jeon_van_roy}, we develop rigorous information-theoretic foundations and use them to derive similar scaling laws.  To keep things simple and concrete, we carry out the analysis with a particular data generating process for which neural networks are well-suited. The sorts of arguments developed by \cite{chinchilla} are just as relevant to this context as they are to language models.

\citet{chinchilla} suggest that the compute optimal trade-off between parameter count and number of training tokens is linear, though the authors expressed some doubt and considered other possibilities that are near-linear as well. We establish an upper bound on the minimal information-theoretically achievable expected error as a function of $p$ and $T$ and derive the relation required to minimize this bound for each compute budget.  For large compute budgets, this relation is linear, as suggested by \cite{chinchilla}.

Our main contributions include a first rigorous mathematical characterization of the compute-optimal efficient frontier for a neural network model and development of information-theoretic tools which enable that.  A limitation of our analysis is in its simplified treatment of computational complexity as the product of the model and data set sizes; we do not assume any constraints on computation beyond those imposed by choices of $p$ and $T$.  In particular, we analyze, algorithms which carry out perfect Bayesian inference with respect to a model that is misspecificified due to its restricted size.  While this abstracts away the details of practical training algorithms, empirical evidence suggests that our idealized framework leads to useful approximations \citep{YifanZhu2022}.  In spite of these limitations, we hope our results set the stage for further mathematical work to guide hyperparameter selection when training large neural networks.

\section{A Framework for Learning}\label{sec:learn_framework}

\subsection{Probabilistic Framework}
We define all random variables with respect to a common probability space $(\Omega, \mathbb{F}, \Pr)$.  Recall that a random variable $F$ is simply a measurable function $\Omega\mapsto\mathcal{F}$ from the sample space $\Omega$ to an outcome set $\mathcal{F}$.

The probability measure $\Pr:\mathbb{F} \mapsto [0,1]$ assigns probabilities to the events in the $\sigma$-algebra $\mathbb{F}$.  For any event $E \in \mathbb{F}$, $\Pr(E)$ to denotes the probability of the event.  For events $E,G\in \mathbb{F}$ for which $\Pr(G) > 0$, $\Pr(E|G)$ to denotes the probability of event $E$ conditioned on event $G$.

For realization $z$ of a random variable $Z$, $\Pr(Z=z)$ is a function of $z$.  We denote its value evaluated at $Z$ by $\Pr(Z)$.  Therefore, $\Pr(Z)$ is a random variable (it takes realizations in $[0,1]$ depending on the value of $Z$).  Likewise for realizations $(y,z)$ of random variables $Y,Z$, $\Pr(Z=z|Y=y)$ is a function of $(y,z)$ and $\Pr(Z|Y)$ is a random variable which denotes the value of this function evaluated at $(Y,Z)$.

If random variable $Z:\Omega\mapsto\Re^K$ has density $p_Z$ w.r.t the Lebesgue measure, the conditional probability $\Pr(E|Z=z)$ is well-defined despite the fact that for all $z$, $\Pr(Z=z) = 0$.  If function $f(z) = \Pr(E|Z=z)$ and $Y:\Omega\mapsto\Re^K$ is a random variable whose range is a subset of $Z$'s, then we use the $\leftarrow$ symbol with $\Pr(E|Z\leftarrow Y)$ to denote $f(Y)$.  Note that this is different from $\Pr(E|Z=Y)$ since this conditions on the event $Z=Y$ while $\Pr(E|Z\leftarrow Y)$ indicates a change of measure.


\subsection{Data}

We consider a stochastic process which generates a sequence $(X_t, Y_{t+1}: t \in \Z_{+})$ of data pairs.  For all $t$, we let $H_t$ denote the history $(X_0, Y_1, \ldots, X_{t-1}, Y_{t}, X_t)$ of experience.  We assume that there exists an underlying latent variable $F$ such that $(X_0, X_1, \ldots) \perp F$ and $F$ prescribes a conditional probability measure $F(\cdot|H_t)$ to the next label $Y_{t+1}$.  In the case of an \emph{iid} data generating process, this conditional probability measure would only depend on $H_t$ via $X_t$.  Note that the current pre-training objective of foundation models falls under this iid setting in which for all $t$, $X_{t}$ is a random segment of the training corpus and $Y_{t+1}$ is the subsequent token.  As our framework is Bayesian, we represent our uncertainty about $F$ by modeling it as a random variable with prior distribution $\Pr(F\in\cdot)$.

\subsection{A Learning Objective}
We focus on a particular notion of error which facilitates analysis via Shannon information theory and reflects the objective of modern foundation models.  For all $t\in \Z_{+}$, our algorithm is tasked with providing a predictive distribution $P_t$ of $Y_{t+1}$ which may depend on the history of data which it has already observed $H_t$.  We express such an algorithm as $\pi$ for which $P_t = \pi(H_t)$.  As aforementioned, an effective learning system ought to leverage data as it becomes available and perform well across all time.  As a result, for any time horizon $T \in \Z_{+}$, we are interested in quantifying the cumulative expected log-loss:
$$\L_{T,\pi} = \frac{1}{T}\sum_{t=0}^{T-1}\E_{\pi}\left[-\ln P_t(Y_{t+1})\right].$$

Note that since we take all random variables to be defined with respect to a common probability space, the expectation $\E$ integrates over all random variables which we do not condition on.  We use the subscript $\pi$ in $\E_{\pi}$ to specify that all predictions $P_t$ for all $t$ are produced by $\pi$.  As $Y_{t+1}$ is the random variable which represents the next label that is generated by the underlying stochastic process, $P_t(Y_{t+1})$ denotes the probability that our algorithm's prediction $P_t$ assigns to label $Y_{t+1}$.

It is important to note that even for an \emph{omniscient} algorithm, the minimum achievable log-loss is not $0$.  Consider the \emph{omniscient} algorithm which produces for all $t$ the prediction $P^*_t = \Pr(Y_{t+1}\in\cdot|F, H_t)$.  Even this agent incurs a loss of:
\begin{align*}
    \frac{1}{T}\sum_{t=0}^{T-1} \E_\pi\left[-\ln \Pr(Y_{t+1}|F, H_t)\right]
    & = \frac{1}{T}\sum_{t=0}^{T-1} \H(Y_{t+1}|F, H_t)\\
\end{align*}
where our point follows from the fact that the conditional entropy $(\H)$ of a discrete random variable $Y_{t+1}$ is non-negative.  As a result, we define the \emph{reducible error} as:
\begin{align*}
    \Lc_{T, \pi}
    & = \L_{T,\pi} - \frac{1}{T}\sum_{t=0}^{T-1}\H(Y_{t+1}|F, H_t)\\
    & = \frac{1}{T}\sum_{t=0}^{T-1}\E\left[\KL\left(P^*_t(\cdot)\| P_t(\cdot)\right)\right].\\
\end{align*}
reducible error represents the error which is reducible via observing additional data and fitting a larger model.  Therefore, we expect that this error will consist of two terms which reflect $1)$ the error due to estimation via finite data, $2)$ the error due to approximation with a finite parameter model.

\section{Error of Constrained Predictors}

We introduce a general upper bound on the reducible error of a \emph{constrained} predictor.  While the formulations remain abstract in this section, a useful running example is the following: Assume that $F$ is an \emph{infinite} width neural network which generates the data and $\tilde{F}$ is a \emph{finite} width network.

\subsection{A Constrained Predictor}

$F$ may exhibit endless complexity, likely beyond what can be represented with finite memory hardware.  To represent the predictions made by a constrained predictor, we first define a random variable $\tilde{F}$ whose range is a subset of $F$'s.  As aforementioned, this random variable can be a lossy compression of $F$ i.e. if $F$ is represented by an infinite-width neural network, $\tilde{F}$ could be a finite-width approximation.  For all $t$, let the constrained predictor be:
$$\tilde{P}_t(\cdot) =\sum_{\tilde{f}} \Pr(\tilde{F}=\tilde{f}|H_t)\cdot\Pr(Y_{t+1}\in\cdot|F=\tilde{f}, X_t).$$
The predictor performs inference on $\tilde{F}$ but performs predictions as if $F=\tilde{F}$.  We let 
$$\Lc_{T}(\tilde{F}) = \frac{1}{T}\sum_{t=0}^{T-1}\E\left[\KL\left(P^*_t(\cdot)\|P_t(\cdot)\right)\right].$$

\subsection{Error of Constrained Predictor}

\begin{restatable}{theorem}{constrained}\label{th:constrained}
    For all $T\in\mathbb{Z}_{++}$ and random variables $F:\Omega\mapsto\mathcal{F}, \tilde{F}:\Omega\mapsto \tilde{\mathcal{F}}$ for which $\tilde{\mathcal{F}}\subseteq\mathcal{F}$, if $((X_t, Y_{t+1}):t\in \Z_{+})$ is iid conditioned on $F$, then
    \begin{align*}
        \tilde{\Lc}_T(\tilde{F})
        & \leq \frac{\I(F;\tilde{F})}{T}+\E\left[\KL\left(P^*_t(\cdot)\|\hat{P}_t(\cdot)\right)\right],
    \end{align*}
    where $\hat{P}_t(\cdot) = \Pr(Y_{t+1}\in\cdot|F\leftarrow\tilde{F}, X_t).$
\end{restatable}

The first term denotes the \emph{estimation error} or the error which is reducible via access to more data.  This is evident by the fact it decreases linearly in $T$ and the numerator reflects the \emph{complexity} of $\tilde{F}$.  The more nats of information that $\tilde{F}$ contains about the data stream, the more data will be required to arrive at a good predictor.

The second term denotes the \emph{misspecification error} or the error which is reducible via a larger learning model.  The closer that $\tilde{F}$ approximates $F$, the smaller the KL divergence between $P^*_t$ and $\hat{P}_t$ will be.  In the following section, we will use Theorem \ref{th:constrained} to derive a concrete neural scaling law for an infinite-width neural network example.

\subsection{Scaling Law}
For a FLOP constraint $C = p\cdot T$, it is clear that there is a tension between $p$ and $T$ in minimizing the upper bound in Theorem \ref{th:constrained}.  This can be seen by first fixing a FLOP count $C$ and substituting $T= C/p$.  The upper bound becomes:
$$\frac{p\cdot \I(F;\tilde{F})}{C} + \E\left[\KL\left(P^*_t(\cdot)\|\hat{P}_t(\cdot)\right)\right].$$
Note that the first term is \emph{increasing} in $p$ whereas the second term is \emph{decreasing} in $p$.  Therefore, under a fixed FLOP budget, the designer ought to select a value of $p$ which effectively balances the two sources of error.

\section{An Illustrative Example}

\subsection{Data Generating Process}

The generating process is described by a neural network with $d$ inputs, a single asymptotically wide hidden layer of ReLU activation units, and a linear output layer.  We denote by $F$ the associated mapping from input to output.  Inputs and binary labels are generated according to $X_t\overset{iid}{\sim}\normal(0, I_d)$ and $\Pr(Y_{t+1} = 1 | F, X_t) = \sigma(F(X_t))$ where $\sigma$ denotes the sigmoid function.

As alluded to by the asymptotic width, $F$ is a nonparametric model which we will outline now.  Let $\bar{\theta}$ be distributed according to a Dirichlet process with base distribution $\mathrm{uniform}(\mathbb{S}^{d-1})$ and scale parameter $K$.  Realizations of this Dirichlet process are probability mass functions on a countably infinite subset of $\sphere$.  Let $\mathcal{W} = \{w \in \mathbb{S}^{d-1} : \bar{\theta}_w > 0\}$ denote this set.  For all $w \in \mathcal{W}$,
$$\theta_w = \left\{\begin{array}{ll}
\overline{\theta}_w \qquad & \text{with probability } 1/2, \\
-\overline{\theta}_w \qquad & \text{otherwise.}
\end{array}\right.$$
Finally, we have that
$$F(X_t)\ =\ \sqrt{K+1}\cdot\sum_{w\in\mathcal{W}}\theta_w\relu\left(w^\top X_t\right).$$
Since $\mathcal{W}$ has countably infinite cardinality, $F$ is characterized by a neural network with infinite width.  We let $\theta = (\theta_w: w\in \mathcal{W})$ and $W = (w: w\in \mathcal{W})$ denote the weights of such neural network and hence
$$F(X_t)\ =\ \sqrt{K+1}\cdot \theta^\top \relu(WX_t).$$
Note that the mean and variance structure satisfy
$$\E[F(X)] = 0, \qquad \E[F(X)^2] = 1/2.$$
Therefore, this model remains nontrivial as $d$ and $K$ grow as all of the above quantities are invariant of $d$ and $K$.

\subsection{Constrained Predictor}
We will study the scaling law associated with a particular constrained predictor characterized by a neural network of width $n$.  Let $\tilde{w}_{1}, \tilde{w}_{2}, \ldots, \tilde{w}_{n}$ be distributed iid ${\rm Categorical}(\bar{\theta})$, where $\mathcal{W}$ are the classes.  For any $\epsilon > 0$, let $\sphere_\epsilon$ be an $\epsilon$-cover w.r.t $\|\cdot \|_2$ and for all $i\in [n]$, let 
$$\tilde{w}_{i,\epsilon} = \argmin_{v \in \sphere_\epsilon} \|\tilde{w}_i-v\|^2_2.$$
Finally, let
$$\tilde{F}_{n,\epsilon}(X_t) = \frac{\sqrt{K+1}}{n}\cdot \sum_{i=1}^{n} {\rm sign}\left(\theta_{\tilde{w}_i}\right)\relu\left(\tilde{w}_{i,\epsilon}^\top X_t\right).$$
Let $\tilde{\theta}\in \Re^n$ is $({\rm sign}(\theta_{\tilde{w}_i})/n: i \in [n])$ and $\tilde{W}_{\epsilon} \in \Re^{n\times d}$ is $(\tilde{w}_{i,\epsilon}: i \in [n])$.  Therefore,
$$\tilde{F}_{n,\epsilon}(X_t) = \sqrt{K+1}\cdot \tilde{\theta}^\top \relu\left(\tilde{W}_\epsilon X_t\right).$$
We consider the performance of a constrained agent which for all $t$, produces the prediction $\tilde{P}_t(\cdot) =$
$$\sum_{\tilde{f}}\ \Pr(\tilde{F}_{n,\epsilon}=\tilde{f}|H_t)\cdot\Pr(Y_{t+1}\in\cdot|F=\tilde{f}, X_t).$$
Note that this agent performs inference on the constrained model $\tilde{F}_{n,\epsilon}$ and produces predictions about $Y_{t+1}$ as if $\tilde{F}_{n,\epsilon}$ were the function $F$ which produced the data.

\subsection{Error Bound}
We will now study the error incurred by the constrained predictor described above.  We define
$$\tilde{\Lc}_{T,n,\epsilon} = \frac{1}{T}\sum_{t=0}^{T-1}\ \E\left[\KL\left(\Pr\left(Y_{t+1}\in\cdot|\theta, X_t\right)\|\tilde{P}_t(\cdot)\right)\right]$$
as the loss of interest.

\begin{restatable}{theorem}{lossUb}\label{th:lossUb}
    For all $n,K,T \in \Z_{++}$ and $\epsilon \geq 0$, if for all $t \in \{0, 1, 2, \ldots, T-1\}$, $(X_t, Y_{t+1})$ is generated by $F$, then 
    \begin{align*}
        \tilde{\Lc}_{T,n,\epsilon}
        & \leq \underbrace{\frac{K\ln\left(1 + \frac{n}{K}\right)\cdot\left(\ln(2n) + d\ln\left(\frac{3}{\epsilon}\right)\right)}{T}}_{\rm estimation\ error} + \underbrace{\frac{3K(1+d\epsilon^2)}{n}}_{\rm misspecification\ error}.
    \end{align*}
\end{restatable}

The estimation error represents the error which is incurred in the process of learning $\tilde{F}$ from $H_T$.  Notably, this error decays linearly in $T$, but only depends \emph{logarithmically} in $n$.  The misspecification error represents the error which persists due to the fact that we approximate $F$ via $\tilde{F}_{n,\epsilon}$.  As a result, this error decreases with greater $n$ and smaller $\epsilon$, but is \emph{independent} of $T$.  If we let $\tilde{\Lc}_{T,n} = \inf_{\epsilon > 0}\ \tilde{\Lc}_{T,n,\epsilon}$, then  

\begin{restatable}{corollary}{optEpsilon}\label{cor:error_bound}
    For all $n\geq 3, K \geq 2, T \in \Z_{++}$, if for all $t \in \{0, 1, 2, \ldots, T-1\}$, $(X_t, Y_{t+1})$ is generated by $F$, then 
    \begin{align*}
        \tilde{\Lc}_{T,n}
        & \leq \frac{dK\ln\left(1+\frac{n}{K}\right)\left(\ln(e36TK)+\frac{2}{d}\ln(2n)\right)}{2T}+ \frac{3K}{n}.
    \end{align*}
\end{restatable}

\begin{figure}[h]
    \centering
    \includegraphics[width=0.55\textwidth]{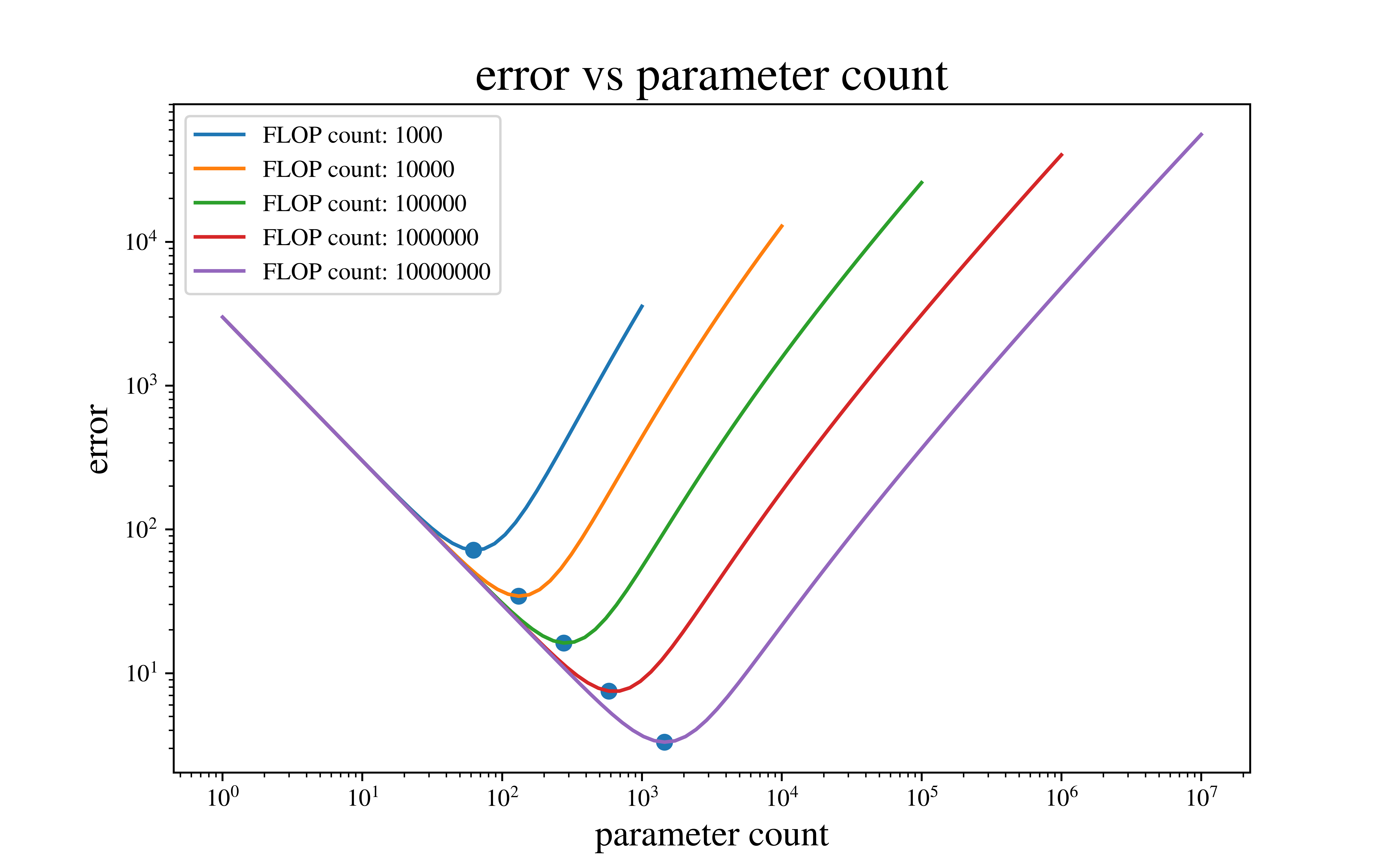}
    \includegraphics[width=0.34\textwidth]{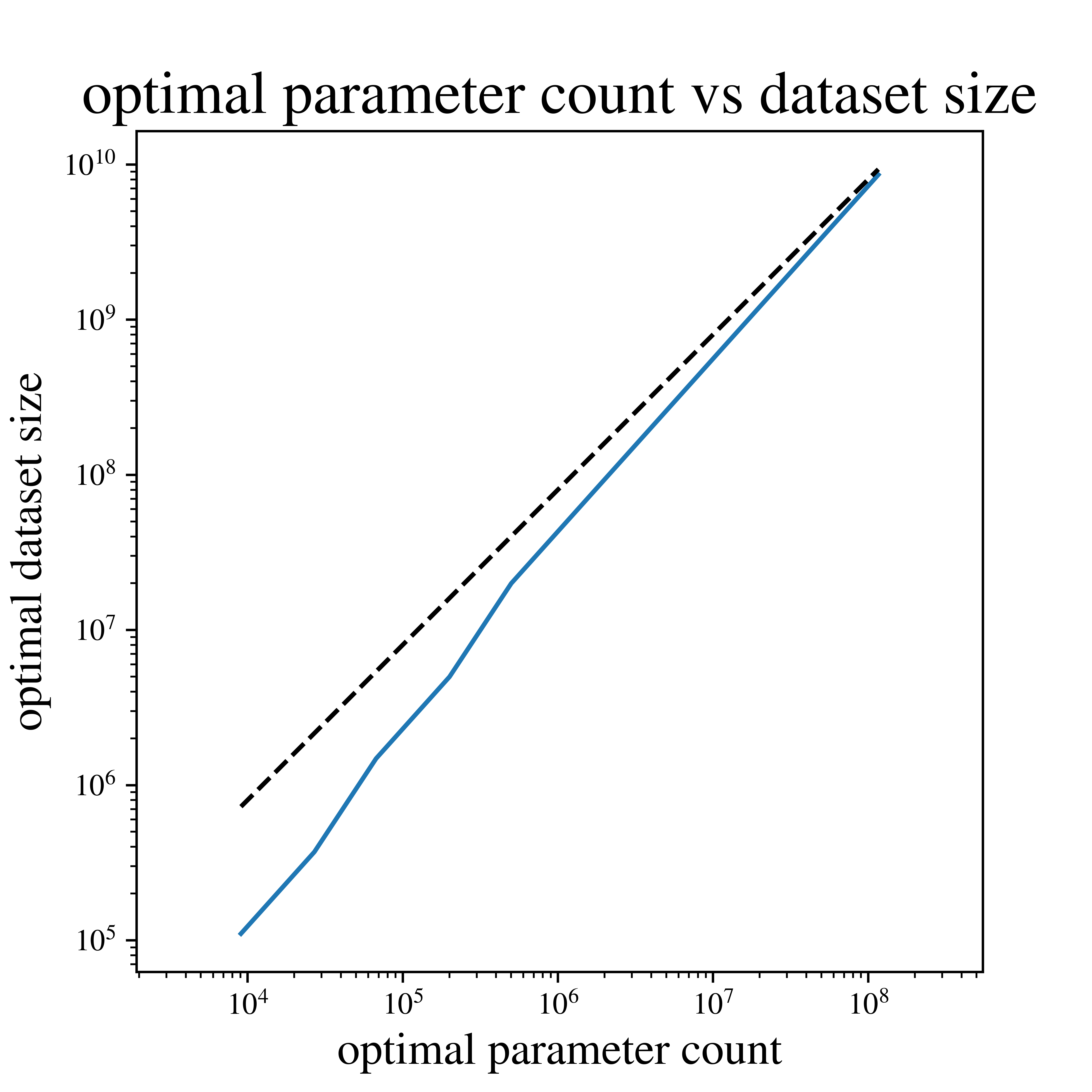}
    \caption{Above (left), we depict the error bound from Corollary \ref{cor:error_bound} for $d=10, K=100$, and various FLOP counts $C$.  Each curve consists of pairs $(n,T)$ for which $d\cdot n\cdot T = C$.  Therefore, each curve depicts the possible error values attainable at a given FLOP count.  Noticeably, an \emph{improper} allocation of compute can lead to \emph{higher} error despite greater resource investment.  (Right) we depict the compute-optimal tradeoff between parameter count and dataset size.  The dashed line represents a line of slope $1$.  As a result, the relationship between optimal parameter count and dataset size eventually looks linear (as suggested by Theorem \ref{th:eff_frontier}).}
    \label{fig:Error}
\end{figure}

\subsection{Resulting Scaling Law}

Corollary \ref{cor:error_bound} provides an upper bound on loss.  We conjecture that this upper bound is tight to within logarithmic factors.  As such, we consider its use as an approximation of loss to guide allocation of compute resources in the following result:

\begin{restatable}{theorem}{effFront}{\bf(compute-optimal parameter count)}
\label{th:eff_frontier}
    For all $d, K\in \mathbb{Z}_{++}$ and FLOP counts $C \in \Z_{++}$, if $K\geq 2, d \geq 3$, and $n^*$ minimizes the upper bound of Corollary \ref{cor:error_bound} subject to $d\cdot n\cdot T \leq C$,
    then
    $$d\cdot n^* = \tilde{\Theta}\left(\sqrt{C}\right).$$
\end{restatable}

A proof of Theorem \ref{th:eff_frontier} can be found in Appendix \ref{apdx:width}.  Recall that the FLOP count $C$ is the product of the parameter count $(n\cdot d)$ and the dataset size $T$.  The above result states that for any FLOP budget $C$, we can minimize the loss upper bound from Corollary \ref{cor:error_bound} by allocating $\tilde{\Theta}(\sqrt{C})$ to the \emph{parameter count}.  This would dictate that the optimal dataset size would \emph{also} be $\tilde{\Theta}(\sqrt{C})$.  This result is consistent with the insights of \cite{chinchilla} that, up to logarithmic factors, the optimal parameter count grows linearly with the training dataset size.

\section{Conclusion}
Our results provide a first step in developing rigorous mathematics for the purposes of analyzing scaling laws for foundation models.  We hope that this will inspire further theoretical research on the subject.  Our analysis is based on an error upper bound and furthermore, our analysis restricts attention to single-hidden-layer feedforward neural networks.  Generalizing the results to treat state-of-the-art architectures remains an open issue.  Furthermore, we have only considered allocation of pretraining compute.  State-of-the-art performance in modern application domains relies on subsequent fine-tuning (see, e.g., \citep{finetuning2019}) through reinforcement learning from human feedback. How best to allocate resources between pretraining and fine-tuning is another area that deserves attention.  An information-theoretic framework that treats pretraining, fine-tuning, and decision making in a unified and coherent manner, perhaps in the vein of \citep{MAL-097}, might facilitate theoretical developments on this front.

\section*{Acknowledgments}
Financial support from the NSF GRFP fellowship and the Army Research Office (ARO) Grant W911NF2010055 is gratefully acknowledged.

\bibliography{references}  

\newpage
\appendix

\section{Proofs of Theoretical Results}

\constrained*
\begin{proof}
    \begin{align*}
        & \tilde{\Lc}_T(\tilde{F})\\
        & = \frac{1}{T}\sum_{t=0}^{T}\E\left[\KL\left(P^*_t(\cdot)\bigg\|\sum_{\tilde{f} \in \tilde{\mathcal{F}}}\Pr(Y_{t+1}\in\cdot|F=\tilde{f}, H_t)\cdot\Pr(\tilde{F}=\tilde{f}|H_t)\right)\right]\\
        & = \frac{1}{T}\sum_{t=0}^{T-1}\I(X_{t+1};F|H_t) + \E\left[\KL\left(\Pr(Y_{t+1}\in\cdot|H_t)\bigg\|\sum_{\tilde{f} \in \tilde{\mathcal{F}}}\Pr(Y_{t+1}\in\cdot|\mathcal{F}=\tilde{f}, H_t)\cdot\Pr(\tilde{F}=\tilde{f}|H_t)\right)\right]\\
        & = \frac{1}{T}\sum_{t=0}^{T-1}\I(X_{t+1};F|H_t)\\
        &\ + \frac{1}{T}\sum_{t=0}^{T-1}\E\left[\sum_{\tilde{f}\in\tilde{\mathcal{F}}}\sum_{y\in \mathcal{Y}}\Pr(Y_{t+1}=y|H_t)\cdot\Pr(\tilde{F}=\tilde{f}|Y_{t+1}=y, H_t)\ln\frac{\Pr(Y_{t+1}=y|H_t)}{\sum_{\tilde{f} \in \tilde{\mathcal{F}}}\Pr(Y_{t+1}=y|F=\tilde{f}, H_t)\cdot\Pr(\tilde{F}=\tilde{f}|H_t)}\right]\\
        & = \frac{1}{T}\sum_{t=0}^{T-1}\I(X_{t+1};F|H_t)\\
        &\ + \frac{1}{T}\sum_{t=0}^{T-1}\E\left[\sum_{\tilde{f}\in\tilde{\mathcal{F}}}\sum_{y\in \mathcal{Y}}\Pr(Y_{t+1}=y|H_t)\cdot\Pr(\tilde{F}=\tilde{f}|Y_{t+1}=y, H_t)\ln\frac{\sum_{\tilde{f}\in\tilde{F}}\Pr(Y_{t+1}=y|H_t)\cdot\Pr(\tilde{F}=\tilde{f}|Y_{t+1}=y, H_t)}{\sum_{\tilde{f} \in \tilde{\mathcal{F}}}\Pr(Y_{t+1}=y|F=\tilde{f}, H_t)\cdot\Pr(\tilde{F}=\tilde{f}|H_t)}\right]\\
        & \overset{(a)}{\leq} \frac{1}{T}\sum_{t=0}^{T-1}\I(X_{t+1};F|H_t)\\
        &\ +\frac{1}{T}\sum_{t=0}^{T-1} \E\left[\sum_{y\in\mathcal{Y}}\sum_{\tilde{f}\in \tilde{\mathcal{F}}} \Pr(Y_{t+1}=y|H_t)\cdot\Pr(\tilde{F}=\tilde{f}|Y_{t+1}=y, H_t)\ln\frac{\Pr(Y_{t+1}=y|H_t)\cdot\Pr(\tilde{F}=\tilde{f}|Y_{t+1}=y, H_t)}{\Pr(Y_{t+1}=y|F=\tilde{f}, H_t)\cdot\Pr(\tilde{F}=\tilde{f}|H_t)}\right]\\
        & = \frac{1}{T}\sum_{t=0}^{T-1}\I(X_{t+1};F|H_t)\\
        &\ + \frac{1}{T}\sum_{t=0}^{T-1}\E\left[\KL\left(\Pr(\tilde{F}\in\cdot|Y_{t+1}, H_t)\|\Pr(\tilde{F}\in\cdot|H_t)\right)\right] + \E\left[\KL\left(\Pr(Y_{t+1}\in\cdot|H_t)\|\Pr(Y_{t+1}\in\cdot|F\leftarrow\tilde{F}, H_t)\right)\right]\\
        & = \frac{1}{T}\sum_{t=0}^{T-1}\E\left[\KL\left(\Pr(Y_{t+1}\in\cdot|F, H_t)\|\Pr(Y_{t+1}\in\cdot| H_t)\right)\right]\\
        & + \frac{1}{T}\sum_{t=0}^{T-1}\I(Y_{t+1};\tilde{F}|H_t) +  \E\left[\KL\left(\Pr(Y_{t+1}\in\cdot|H_t)\|\Pr(Y_{t+1}\in\cdot|F\leftarrow\tilde{F}, H_t)\right)\right]\\
        & = \frac{1}{T}\sum_{t=0}^{T-1}\I(Y_{t+1};\tilde{F}|H_t)  + \E\left[\KL\left(\Pr(Y_{t+1}\in\cdot|F, H_t)\|\Pr(Y_{t+1}\in\cdot|F\leftarrow \tilde{F}, H_t)\right)\right]\\
        & = \frac{\I(H_T;\tilde{F})}{T} + \frac{1}{T}\sum_{t=0}^{T-1}\E\left[\KL\left(\Pr(Y_{t+1}\in\cdot|F, H_t)\|\Pr(Y_{t+1}\in\cdot|F\leftarrow \tilde{F}, H_t)\right)\right],
    \end{align*}
    where $(a)$ follows from the log-sum inequality.
\end{proof}

\subsection{Proof of Dirichlet Process Results}

\begin{lemma}\label{le:kl_ub}{\bf(squared error upper bounds KL)}
    For all real-valued random variables $G$ and $\tilde{G}$, if $Y$ is a binary random variable for which
    $\Pr(Y=1|G) = \frac{1}{1+e^{-G}}$, then 
    $$\E\left[\KL(\Pr(Y\in\cdot|G)\| \Pr(Y\in\cdot|G\leftarrow\tilde{G}))\right] \leq \E\left[\left(G-\tilde{G} \right)^2\right].$$
\end{lemma}
\begin{proof}
    \begin{align*}
        \E\left[\KL(\Pr(Y\in\cdot|G)\| \Pr(Y\in\cdot|G\leftarrow\tilde{G}))\right]
        & = \E\left[\frac{1}{1+e^{G}}\ln\left(\frac{1+e^{\tilde{G}}}{1+e^{G}}\right)\right]\\
        & \quad + \E\left[\frac{1}{1+e^{-G}}\ln\left(\frac{1+e^{-\tilde{G}}}{1+e^{-G}}\right)\right]\\
        & \overset{(a)}{\leq} \E\left[\left(G - \tilde{G}\right)^2\right]\\
    \end{align*}
    where $(a)$ follows from the fact that for all $x, y \in \Re$, $\frac{1}{1+e^x}\ln\left(\frac{1+e^y}{1+e^x}\right) + \frac{1}{1+e^{-x}}\ln\left(\frac{1+e^{-y}}{1+e^{-x}}\right) \leq (x-y)^2$.
\end{proof}

\begin{lemma}\label{le:dir_mult_ub}
    For all $d, n, N\in \mathbb{Z}_{++}$, if $X\sim\normal(0, I_d)$, then
    $$\E\left[\left(F(X) - \tilde{F}_{n,0}(X)\right)^2\right] \leq \frac{K+1}{n}.$$
\end{lemma}
\begin{proof}
    \begin{align*}
        \E\left[\left(\theta^\top \relu(WX) - \tilde{\theta}^\top\relu(\tilde{W}X)\right)^2\right]
        & \overset{(a)}{\leq} \frac{K+1}{n^2}\cdot\E\left[\relu(\tilde{W}X)^\top\left(\proxytheta \proxytheta^\top\right)\relu(\tilde{W}X)\right]\\
        & = \frac{K+1}{n^2}\cdot\E\left[\relu(\tilde{W}X)^\top I_n\relu(\tilde{W}X)\right]\\
        & \leq \E\left[\frac{K+1}{n^2}\cdot\sum_{i=1}^{n} (\tilde{w}_{i,0}^\top X)^2\right]\\
        & = \frac{K+1}{n}.
    \end{align*}
    where $(a)$ follows from the fact that the two functions have equal conditional expectation conditioned on $\theta$.
\end{proof}

\begin{lemma}\label{le:miss_ub_1}
    For all $d, n, K\in\mathbb{Z}_{++}$,
    $$\E\left[\KL\left(\Pr(Y\in\cdot|F, X)\|\Pr(Y\in\cdot|F\leftarrow\tilde{F}_{n,0}, X)\right)\right]\ \leq\ \frac{K+1}{n}.$$
\end{lemma}
\begin{proof}
    \begin{align*}
        \E\left[\KL(\Pr(Y\in\cdot|F, X)\|\Pr(Y\in\cdot|F\leftarrow \tilde{F}_{n,0}, X))\right]
        & \overset{(a)}{=} \E\left[\left(F(X) - \tilde{F}_{n,0}(X)\right)^2\right]\\
        & \overset{(b)}{\leq} \frac{K+1}{n},\\
    \end{align*}
    where $(a)$ follows from Lemma \ref{le:kl_ub}, $(c)$ follows from the fact that the distribution of $\theta$ is the limiting distribution $\lim_{N\rightarrow\infty}$ of a Dirichlet $[K/N,\ldots, K/N]$ random variable, $(d)$ follows from the dominated convergence theorem, and $(e)$ follows from Lemma \ref{le:dir_mult_ub}. 
\end{proof}

\begin{lemma}\label{le:miss_ub_2}
    For all $d, n, K\in\mathbb{Z}_{++}$,
    $$\E\left[\left(F_{n,0}(X) - F_{n,\epsilon}(X)\right)^2\right]\ \leq\ \frac{d(K+1)\epsilon^2}{n}.$$
\end{lemma}
\begin{proof}
    \begin{align*}
        \E\left[\left(F_{n,0}(X) - F_{n,\epsilon}(X)\right)^2\right]
        & = \E\left[(K+1)\cdot\left\|\sum_{i=1}^{n}\tilde{\theta}_i \left(\relu(\tilde{w}_{i,0}^\top X) - \relu(\tilde{w}_{i,\epsilon}^\top X)\right)\right\|^2\right]\\
        & = \E\left[(K+1)\cdot\sum_{i=1}^{n}\tilde{\theta}_i^2\cdot \left\| \left(\relu(\tilde{w}_{i,0}^\top X) - \relu(\tilde{w}_{i,\epsilon}^\top X)\right)\right\|^2\right]\\
        & \leq \frac{d(K+1)\epsilon^2}{n}
    \end{align*}
\end{proof}

\begin{lemma}\label{le:miss_ub_3}
    For all $d, n, K\in\mathbb{Z}_{++}$ and $\epsilon \geq 0$,
    $$\E\left[\KL\left(\Pr(Y\in\cdot|F, X)\|\Pr(Y\in\cdot|F\leftarrow\tilde{F}_{n,\epsilon}, X)\right)\right]\ \leq\  \frac{3K(1+d\epsilon^2)}{n}.$$
\end{lemma}
\begin{proof}
    \begin{align*}
        \E\left[\KL\left(\Pr(Y\in\cdot|F, X)\|\Pr(Y\in\cdot|F\leftarrow\tilde{F}_{n,\epsilon}, X)\right)\right]
        & \overset{(a)}{\leq} \E\left[\left(F(X) - \tilde{F}_{n,\epsilon}(X)\right)^2\right]\\
        & = \E\left[2\left(F(X) - F_{n,0}(X)\right)^2 + 2\left(F_{n,0}(X) - F_{n,\epsilon}(X)\right)^2\right]\\
        & \overset{(b)}{\leq} \frac{2(K+1)}{n} + \frac{2d(K+1)\epsilon^2}{n}\\
        & \leq \frac{3K(1+d\epsilon^2)}{n},
    \end{align*}
    where $(a)$ follows from Lemma \ref{le:kl_ub} and $(b)$ follows from Lemmas \ref{le:miss_ub_2} and \ref{le:miss_ub_3}.
\end{proof}

\begin{lemma}{\bf(entropy upper bound)}
    \label{le:rate_ub}
    For all $d, n, K \in \mathbb{Z}_{++}$ and $\epsilon > 0$,
    $$\H(\tilde{F}_{n,\epsilon}) \leq K\ln\left(1 + \frac{n}{K}\right)\cdot\left(\ln(2n) + d\ln\left(\frac{3}{\epsilon}\right)\right).$$
\end{lemma}
\begin{proof}
    \begin{align*}
        \H(\tilde{F}_{n,\epsilon})
        & \overset{(a)}{\leq} \E\left[\sum_{w\in\mathcal{W}} \mathbbm{1}_{|\proxytheta_w|>0}\cdot\left(\ln(2n) + d\ln\left(\frac{3}{\delta}\right)\right)\right]\\
        & \overset{(b)}{\leq} K\ln\left(1+\frac{n}{K}\right)\left(\ln(2n) + d\ln\left(\frac{3}{\delta}\right)\right)
    \end{align*}
    where $(a)$ follows from the fact that the output weight can take on at most $2n$ different values $(-\frac{n\sqrt{K+1}}{n}, -\frac{(n-1)\sqrt{K+1}}{n},\ldots,-\frac{\sqrt{K+1}}{n}, \frac{\sqrt{K+1}}{n},\ldots \frac{n\sqrt{K+1}}{n})$ and the fact that $|\sphere_\epsilon| \leq (3/\delta)^d$ and $(b)$ follows as a commonly known fact about the number of unique classes of a dirichlet-multinomial distribution.
\end{proof}

\lossUb*
\begin{proof}
    The result follows from Theorem \ref{th:constrained} and Lemmas \ref{le:rate_ub} and \ref{le:miss_ub_3}
\end{proof}

\optEpsilon*
\begin{proof}
    The result holds from Theorem \ref{th:lossUb} by setting $\epsilon^2 = \frac{nK\ln(1+\frac{n}{K})}{4T(K+1)}$ and the fact that for $n \geq 3, K \geq 2,$ $\frac{36T(K+1)}{nK\ln(1+n/K)}\leq 36KT$.
\end{proof}

\subsection{Optimal Width}\label{apdx:width}
\effFront*
\begin{proof}
    \begin{align*}
        n^*
        & = \argmin_{n\in\left[\frac{C}{d}\right]}\frac{3K}{n}+\frac{K\ln\left(1+\frac{n}{K}\right)\cdot\ln\left(2n\right)}{t} + \frac{dK\ln\left(1+\frac{n}{K}\right)\left(1 + \frac{1}{2}\ln\left(36KT\right)\right)}{t};\ \text{ s.t. } n\cdot d\cdot t \leq C\\
        & = \argmin_{n\in\left[\frac{C}{d}\right]}\frac{3}{n}+\frac{\ln\left(1+\frac{n}{K}\right)\cdot\ln\left(2n\right)}{t} + \frac{d\ln\left(1+\frac{n}{K}\right)\left(1 + \frac{1}{2}\ln\left(36KT\right)\right)}{t};\ \text{ s.t. } n\cdot d\cdot t \leq C\\
        & = \argmin_{n\in\left[\frac{C}{d}\right]}\frac{3}{n}+\frac{nd\ln\left(1+\frac{n}{K}\right)\cdot\ln\left(2n\right)}{C} + \frac{nd^2\ln\left(1+\frac{n}{K}\right)\left(1 + \frac{1}{2}\ln\left(\frac{36KC}{nd}\right)\right)}{C}\\
        & \overset{(a)}{=} n \text{ s.t. } \frac{3}{n^2} = \frac{d\left(\ln\left(1+\frac{n}{K}\right) + \ln\left(1+\frac{n}{K}\right)ln(2n) + \frac{n}{K+n}\ln\left(2n\right)\right)}{C}\\
        &\qquad +\frac{d^2\ln\left(1+\frac{n}{K}\right)\ln\left(\frac{36KC}{nd}\right)}{2C}+ \frac{d^2n}{C(n+K)} + \frac{nd^2\ln\left(\frac{36KC}{nd}\right)}{2C\left(n+K\right)}\\
        & = n \text{ s.t. } C = \frac{dn^2\left(\ln\left(1+\frac{n}{K}\right) + \ln\left(1+\frac{n}{K}\right)ln(2n) + \frac{n}{K+n}\ln\left(2n\right)\right)}{3}\\
        &\qquad +\frac{d^2n^2\ln\left(1+\frac{n}{K}\right)\ln\left(\frac{36KC}{nd}\right)}{6}+ \frac{d^2n^3}{3(n+K)} + \frac{n^3d^2\ln\left(\frac{36KC}{nd}\right)}{6\left(n+K\right)}.\\
    \end{align*}
    where $(a)$ follows from $1$st order optimality conditions

    Due to monotonicity, we can drive upper and lower bounds for the value of $n$ via lower and upper bounds of the above RHS respectively.

    We begin with the upper bound for $n^*$:
    \begin{align*}
        n^* & \overset{(a)}{\leq} n \text{ s.t. } C = \frac{d^2n^2}{3}\\
        & = \frac{\sqrt{3C}}{d}.
    \end{align*}
    where $(a)$ follows from the fact that $d^2n^2/3$ is a lower bound of the above RHS.

    We now derive the lower bound for $n^*$:
    \begin{align*}
        n^* & \overset{(a)}{\geq} n \text{ s.t. } C = dn^2\ln(2n)\ln(36KC)\\
        & = \tilde{\Omega}\left(\frac{\sqrt{C}}{d}\right). 
    \end{align*}
    where $(a)$ follows from the fact that $dn^2\ln(2n)\ln(36KC)$ is an upper bound of the above RHS. The result follows.
\end{proof}

\end{document}